\documentclass[final,5p,times]{elsarticle}

\usepackage{lineno,hyperref}
\modulolinenumbers[5]

\journal{Journal of \LaTeX\ Templates}

\usepackage[ruled,linesnumbered]{algorithm2e}

\usepackage{amsmath}
\usepackage{amsfonts}
\usepackage{amsthm}
\usepackage{arydshln}
\usepackage{booktabs}
\usepackage{epstopdf}
\usepackage{enumerate}
\usepackage{graphicx}
\usepackage{lineno}
\usepackage{multirow}
\usepackage{pdfpages}
\usepackage{subfigure}
\usepackage{tabularx}

\graphicspath{{figs/}}

\DeclareMathOperator*{\argmax}{arg\,max}
\DeclareMathOperator*{\argmin}{arg\,min}
\newtheorem{prop}{Proposition}

\theoremstyle{thmstyleone}

\theoremstyle{thmstyletwo}

\theoremstyle{thmstylethree}

\begin{document}

\begin{frontmatter}

\title{Interpreting Vulnerabilities of Multi-Instance Learning to Adversarial Perturbations}


\author[1]{Yu-Xuan Zhang}
\ead{inki.yinji@gmail.com}

\author[1]{Hua Meng}
\ead{menghua@swjtu.edu.cn}

\author[2]{Xue-Mei Cao}
\ead{caoxuemei.qpz@gmail.com}

\author[1]{Zhengchun Zhou\corref{mycorrespondingauthor}}
\ead{zzc@swjtu.edu.cn}
\cortext[mycorrespondingauthor]{Corresponding author}

\author[3]{Mei Yang}
\ead{yangmei@swpu.edu.cn}

\author[1]{Avik Ranjan Adhikary}
\ead{a.adhikary@swjtu.edu.cn}

\address[1]{School of Information Science and Technology, Southwest Jiaotong University, Chengdu 611730, China}
\address[2]{School of Computing and Artifical Intelligence, Southwestern University of Finance and Economics, Chengdu, 611130, China}
\address[3]{School of Computer Science, Southwest Petroleum University, Chengdu 610500, China}

\begin{abstract}
Multi-instance Learning (MIL) is a recent machine learning paradigm which is immensely useful in various real-life applications, like image analysis, video anomaly detection, text classification, etc.
It is well known that most of the existing machine learning classifiers are highly vulnerable to adversarial perturbations.
Since MIL is a weakly supervised learning, where information is available for a set of instances, called bag and not for every instances, adversarial perturbations can be fatal.
In this paper, we have proposed two adversarial perturbation methods to analyze the effect of adversarial perturbations to interpret the vulnerability of MIL methods.
Out of the two algorithms, one can be customized for every bag, and the other is a universal one, which can affect all bags in a given data set and thus has some generalizability.
Through simulations, we have also shown the effectiveness of the proposed algorithms to fool the state-of-the-art (SOTA) MIL methods.
Finally, we have discussed through experiments, about taking care of these kind of adversarial perturbations through a simple strategy.
\emph{Source codes are available at \href{https://github.com/InkiInki/MI-UAP}{https://github.com/InkiInki/MI-UAP}}.
\end{abstract}


\end{frontmatter}


\section{Introduction}\label{sec: introduction}

The data structure of multi-instance learning (MIL) makes it unique by describing each sample as a bag of many instances.
In MIL, the number of instances is typically higher than that of bags, while the supervised information is only provided at the bag-level.
According to the basic MIL assumption \cite{Dietterich:1997:3171}, a bag is positive if it contains at least one positive instance; otherwise it is negative.
Such weakly supervised scenarios are challenging since they only offer a very small amount of information to support decision-making.
But, this makes MIL more attractive to the researchers, which enables them to work at a higher level (bag-level) and address many real-world issues, which includes \cite{Lin:2022:interventional,Liu:2018:77277735,Qin:2022:multi,Yang:2022:109121,Zeng:2022:110}, drug production \cite{Dietterich:1997:3171,Kwok:2007:901906,Wu:2018:10651080}, and video abnormal detection \cite{Li:2022:self,Nguyen:2018:67526761,Sultani:2018:64796488}.

MIL models can be broadly classified into two categories \cite{Li:2021:1431814328}:
embedding- and instance-based methods.
The basic difference between these two approaches is that, the former obtains the label of the bag, end-to-end, without the need to provide or learn an instance-level classifier.
As a result, this type of method has gained popularity among researchers in recent times, which gives rise to two main ideas: neural network \cite{Campanella:2019:13011309,Ilse:2018:21272136,Konstantinov:2022:123,Wang:2018:1524} and traditional embedding methods \cite{Wei:2017:975987,Yang:2021:112}.
The main task of these methods is to learn the bag embedding through a permutation-invariant transformation, so the quality of the embedding directly affects the classification performance.
However, these methods hardly take algorithm security into account, which makes the algorithm seriously vulnerable to adversarial perturbations, which can fool image search engines or surveillance cameras.

\begin{figure}[!htb]
\centering
\includegraphics[width=\linewidth]{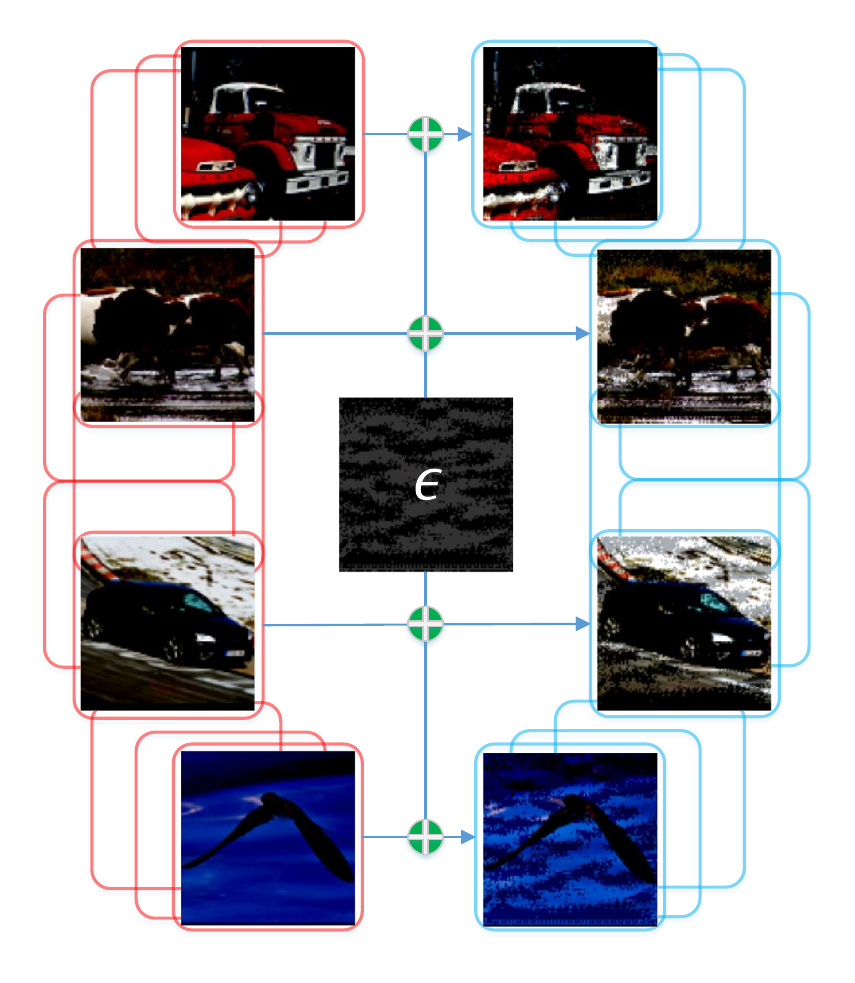}
\caption{
A demonstration of the proposed MI-UAP applied to ABMIL \cite{Ilse:2018:21272136} on the STL10 data set \cite{Coates:2011:215223}.
The bags, universal perturbation $\epsilon$, and perturbed bags are depicted in the left, central, and right images, respectively.
The objective is to determine whether a bag contains instances (images) labeled as trucks.
The MIL learner may, however, turn a blind eye to the fact that the bag contains at least one truck when the perturbation is added.
}
\label{fig: example}
\end{figure}

To interpret the vulnerability of MIL algorithm, in this paper, we propose multi-instance customized and universal adversarial perturbation (MI-CAP and MI-UAP), a demonstration is shown in Figure \ref{fig: example}.
For a given bag containing instances, we have proposed two ways to generate perturbations.
First, MI-CAP is used to fool the prediction of the MIL learner of the current bag by iterating the gradient-related computation.
The perturbation that is created for each bag may be unique, which is also the meaning of ``customized''.
In the next method, MI-UAP obtains a universal perturbation that is easy to store and can be easily generalized by repeatedly fine-tuning the perturbation generated by MI-CAP.
While such a process will result in some performance losses, since it needs to store one perturbation, it is very advantageous when dealing with data sets.
To verify the effectiveness of these methods, we have conducted experiments on three well-known image databases, five classic MIL data sets, and two popular video anomaly detection data sets.
The fooled methods are all state-of-the-art (SOTA) MIL methods, including five neural networks and five traditional methods.
The experimental results expose the vulnerability of these methods, which can be fatal in some circumstances, especially when the bag actually does include instances of our demand.
Hence, we also propose a simple experiment to illustrate the prevention of the proposed adversarial perturbations, and the results show that it can somewhat offset the algorithm's vulnerability.

In summary, our main contributions are as follows:

\begin{itemize}
  \item
  We design two strategies to generate perturbations to expose the vulnerability of SOTA MIL methods, one of which is bag-customized and the other indirectly demonstrates the existence of universal bag-agnostic perturbations for these MIL methods.
  \item
  The proposed methods and the generated adversarial bags can be considered useful knowledge or prior experience for future learners.
  The reason is that MI-CAP can generate perturbations that are unaffected by other bags, while the perturbations generated by MI-UAP can be easily stored as part of the knowledge base.
  \item
  We demonstrate the generalizability of these perturbations in experiments, showing that perturbations generated by one neural network can fool other networks as well.
  In addition, we propose a simple strategy to mitigate the effect of these perturbations as much as possible.
\end{itemize}

\section{Related Work}

In this section, we revisit some works related to MIL methods and adversarial perturbation approaches.

\subsection{MIL Methods}

In his study in 1997 about drug activity prediction \cite{Dietterich:1997:3171}, Dietterich \textit{et al.} proposed MIL.
MIL is currently used in many real-world applications, such as image classification and video anomaly detection due to its unique features, and it is can be broadly classified in two branches.
For traditional methods, MILES \cite{Chen:2006:19311947} builds a permutation-invariant transformation that embeds the bag into a new feature space for instance-level classifiers to function.
BAMIC and ELDB \cite{Zhang:2009:4768,Yang:2021:112} use the set-kernel as the foundation, and each feature of the bag embedding vector denotes how similar it is to the corresponding key bag.
miVLAD and miFV \cite{Wei:2017:975987} are very similar to BAMIC and ELDB, with the primary distinction being that the instance-set distance metric takes the place of the set-kernel.
Additional methods of this type include BDR \cite{Huang:2022:108583}, miVLAD \cite{Wei:2017:975987}, ISK \cite{Xu:2019:941949}, AEMI \cite{Yang:2022:109121}, MSK \cite{Yang:2022:339351}.

At the same time, MIL neural networks also advanced quickly, particularly the more recent attention-based ones.
ABMIL \cite{Ilse:2018:21272136} presents two attention strategies by expressing the MIL problem as learning the Bernoulli distribution of the bag label, where the bag label probability is completely parameterized by neural networks.
LAMIL \cite{Shi:2020:57425745} proposes a new loss function based on the attention mechanism, which completes the prediction of the bag label while calculating the weight and uses the consistency cost to boost the generalization performance of the model.
DSMIL \cite{Li:2021:1431814328} introduces a novel MIL aggregator to model the relationships between the instances in the bag and employs self-supervised contrastive learning to extract the good representation of bags.
It also adopts a pyramidal mechanism to further enhance classification performance.
MAMIL \cite{Konstantinov:2022:123} takes into account nearby instances of each analysis instance to efficiently process different types of instances and implement a diverse feature representation of the bag.
Some other methods include MILRNN \cite{Campanella:2019:13011309}, C2C \cite{Sharma:2021:682698}, BSN \cite{Wang:2019:578588}, DP-MINN \cite{Yan:2018:662677} and BP-MIP \cite{Zhang:2004:110}.

\subsection{Adversarial Perturbation Approaches}

Numerous methods have demonstrated the vulnerability of traditional machine learning or deep learning algorithms.
DeepFool \cite{Moosavi:2016:25742582} shows how deep networks can be fooled with only minor, well-sought perturbations to images.
UAP \cite{Moosavi:2017:17651773} proves the existence of an image-agnostic and tiny perturbation vector that causes natural images to be misclassified frequently.
\cite{Xu:2020:151178} presents a methodical and comprehensive overview of the main threats of attacks for three most popular data types, including images, graphs and text.
Homotopy-Attack \cite{Zhu:2021:1286812877} fools the neural networks by jointly addressing the sparsity and the perturbation bound in one unified framework.
UAP-GAN \cite{Wang:2022:122} respectively studies the robustness of image classification and captioning systems based on convolutional neural networks and recurrent neural networks.
LFT \cite{Zhao:2022:17} leverages bilevel optimization and learning-to-optimize techniques for perturbation generation with an improved attack success rate.

In summary, MIL is a learning form with a special data structure that has derived various learning strategies and achieved good results in several real-life applications, like image classification, video anomaly detection, and others. However, since MIL is a special kind of machine learning, which is vulnerable to adversarial perturbations, the security and vulnerabalities of MIL algorithms need to be discussed.

\section{Methodology}

In this section, we will firstly present the problem statement and then propose two perturbation strategies for MIL, namely the customized and universal adversarial perturbations.

\subsection{Problem Statement}

In the MIL scenario, each bag $B_i = \{ \mathbf{x}_{i1}, \mathbf{x}_{i2}, \dots, \mathbf{x}_{in_i} \}$ corresponds to a sample, where $\mathbf{x}_{ij} \in \mathbb{R}^d, j = 1, 2, \dots, n_i$ is the $j$-th instance of $B_i$, $n_i$ is the cardinality of $B_i$, and $d$ is the dimension.
Based on the basic MIL assumption \cite{Dietterich:1997:3171}, the supervised information $y_i \in Y = \{ 1, 0 \}$ of the label of $B_i$ is only given at the bag-level, while instance labels are either unknown or unavailable.
In addition, a bag is positive ($y_i = 1$) if it contains at least one positive instance;
otherwise negative ($y_i = 0$).
Each instance $\mathbf{x}_{ij}$ can have an assumed label $y_{ij} \in Y$ from which the bag's label can be deduced:
\begin{equation}\label{eq: mi_assumption}
y_i = \left\{
\begin{array}{lll}
1, & \forall j \in 1, 2, \dots, n_i, \exists y_{ij} = 1;\\
0, & \text{otherwise}.
\end{array}
\right.
\end{equation}

One of the most popular MIL research topics is neural network methods.
Take as an example these methods \cite{Ilse:2018:21272136,Konstantinov:2022:123,Shi:2020:57425745}, which aim to learn an end-to-end mapping $\mathcal{F}: B_i \to \hat{y}_i = \sigma(\mathbf{b}_i) = \argmax_c p_{ic}$ to embed each bag into a new feature space while keeping the highest degree of distinguishability to maximize the classification performance, where $p_{ic} \in \mathbf{p}_i = f (\mathbf{b}_i)$ is the probability that the bag is from class $c$ and $f (\cdot)$ a fully connected layer as well as $\mathbf{b}_i \in \mathbb{R}^{d'}$ is the embedding vector of $B_i$.

However, they are easily fooled by a slight perturbation $\epsilon$ into excluding a bag from the decision region it belongs to.
This is exactly what we intend to do, i.e., expose the MIL algorithms' vulnerability by maximizing the fool rate:
\begin{equation}\label{eq: goal}
\max_{B_i \in \mathcal{S}} \frac{\sum_i \mathbb{I} (\tilde{y}_i \neq \hat{y}_i)}{\sum_i \mathbb{I} (y_i == \hat{y}_i)},
\end{equation}
where $\mathcal{S} = \{B_i\}_{i = 1}^N$ is the data set with the size $N$, $\epsilon$ is a vector of the same dimension as $\mathbf{x}_{ij}$, and $\mathbb{I}(\cdot)$ is an indicator function that outputs $1$ if the input is true, otherwise $\mathbb{I}(\cdot)=0$.
Here $\tilde{y}_i = \mathcal{F} (B_i \oplus \epsilon)$, and $B_i \oplus \epsilon$ refers to adding interference to each instance in the bag, which is equivalent to $\{\mathbf{x}_{ij} + \epsilon\}_{j = 1}^{n_i}$.
This has two advantages:
1) ``Usually unequal number of instances in different bags'' is not a restriction; and
2) The number of positive instances is usually much smaller than the number of negative ones, while adding perturbations to positive ones takes extra time.
This problem is avoided by Eq. \eqref{eq: goal}.
Next, we will focus on two perturbation generation methods.

\subsection{Customized Adversarial Perturbation (CAP)}

\begin{algorithm}[!htb]
\caption{MI-CAP}\label{ag: cap}
 \KwData{
 $\ $\\
 $\quad$Bag $B_i$;\\
 $\quad$MIL neural network $\mathcal{F}$;\\
 $\quad$Number of iterations $L_1$;\\
 $\quad$Gradient computation mode: ``ave'' or ``att'';}
 \KwResult{
 $\ $\\
 $\quad$Customized perturbation $\epsilon_i$;
 }
 Initialize the perturbation $\epsilon_i^l = \mathbf{0}$\ and the gradient vector $\mathbf{g}_i^l = \mathbf{0}$, where $l = 0$\;
 Get the prediction $\tilde{y}_i = \mathcal{F} (B_i + \epsilon_i^0)$\;
 \While{$\mathcal{F} (B_i \oplus \epsilon_i^l) == \mathcal{F}(B_i)$ and $l < L_1$ }{
  $l$++\;
  Compute the gradient $\hat{G}_{i} = \frac{\partial p_{i\tilde{y}_i}}{\partial (B_i \oplus \epsilon_i^{l - 1})}$\;
  Compute the gradient $\tilde{G}_i = \frac{\partial p_{i\tau_i}}{\partial (B_i \oplus \epsilon_i^{l - 1})}$, where $\tau_i =$ Eq. \eqref{eq: min_perturbation}\;
  \eIf{``ave'' is used}{
    $\mathbf{g}_i^{l-1} = \frac{1}{n_i} \sum_{j = 1}^{n_i} \hat{\mathbf{g}}_j$, where $\hat{\mathbf{g}}_j$ is the $j$-th row of $\hat{G}_i$\;
    $\mathbf{g}_i^l = \frac{1}{n_i} \sum_{j = 1}^{n_i} \tilde{\mathbf{g}}_j$, where $\tilde{\mathbf{g}}_j$ is the $j$-th row of $\tilde{G}_i$\;
  }
  {
    $\mathbf{g}_i^{l-1} = \hat{\mathbf{g}}_{\iota_i}$, where $\iota_i =$ Eq. \eqref{eq: ins_idx}\;
    $\mathbf{g}_i^{l} = \tilde{\mathbf{g}}_{\iota_i}$\;
  }
  $\Delta \epsilon_i = \mathbf{g}_i^l - \mathbf{g}_i^{l - 1}$\;
  $\Delta \epsilon_i \leftarrow |p_{i\tau_i} - p_{i\hat{y}_i}|\frac{\Delta \epsilon_i}{\| \Delta \epsilon_i + \eta\|_2^2}$, where $\eta > 0$ is a small number used to avoid division by $0$\;
  $\epsilon_i^l = \epsilon_i^{l - 1} + \Delta \epsilon_i$\;
  Update $\tilde{y}_i$ via $\mathcal{F} (B_i \oplus \epsilon_i^l)$\;
  }
  \Return $\epsilon_i^l$\;
\end{algorithm}

It is not uncommon to perturb a specific sample \cite{Goodfellow:2014:111,Moosavi:2016:25742582,Xu:2020:151178}, and there is a clear advantage to doing so:
Ideally, a particular perturbation can always be generated to lead the classifier to misclassify the current sample.
Like their basic idea,
the goal of multi-instance learning customized adversarial perturbation (MI-CAP) is to fool the MIL learner by generating a unique perturbation $\epsilon_i$ for $B_i$ over a number of iterations $L$.
Such a process is called ``customized'' because the same perturbations may not exist at all for a given data set.
Specifically, the minimal perturbation that fools the classifier to mispredict the label of $B_i$ can be written as
\begin{equation}\label{eq: min_perturbation}
\argmin_{\epsilon_i} \| \epsilon_i \|_2\quad\text{s.t.}\quad p_{i\tau_i}^{cus} \geq \hat{y}_i, \tau_i = Y \setminus \{ \hat{y}_i \}
\end{equation}
where $p_{ic}^{cus} \in l(\mathbf{b}_i^{cus})$ and $\mathbf{b}_i^{cus}$ is the embedding vector of the perturbed bag $B_i \oplus \epsilon_i$.
Algorithm \ref{ag: cap} gives the solution to this equation based on the core idea of DeepFool \cite{Moosavi:2016:25742582}, which can handle non-convex optimization problems brought by neural networks as the classifier.

\begin{figure*}[!htb]
    \centering
    \subfigure[MI-CAP]{\includegraphics[width=0.44\hsize]{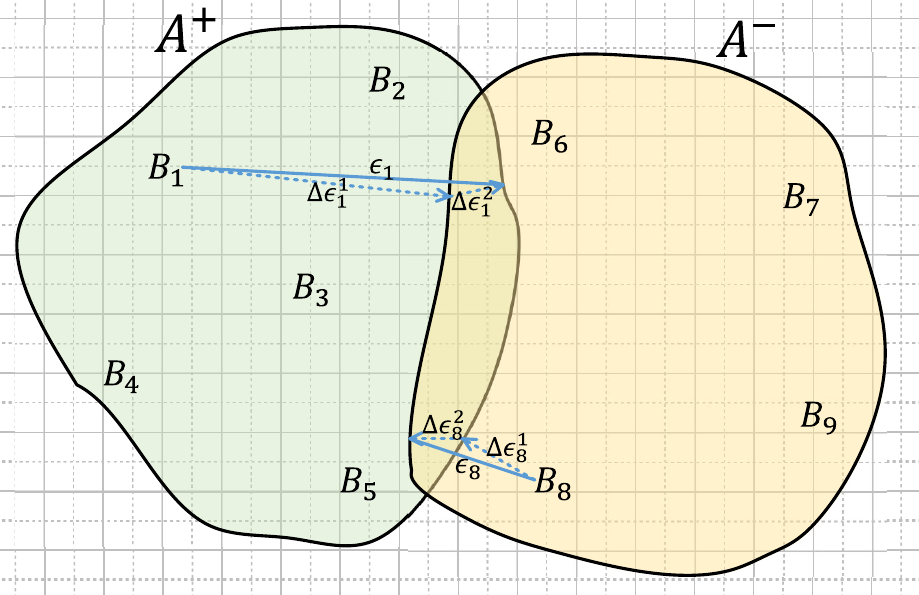}\label{fig: cap}}
    \hspace{0.05\hsize}
    \subfigure[MI-UAP]{\includegraphics[width=0.44\hsize]{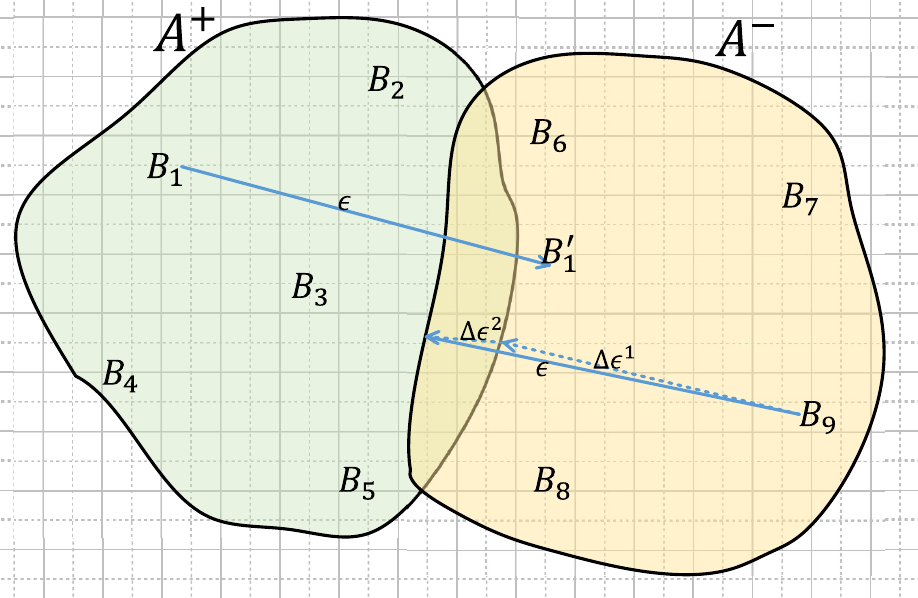}\label{fig: uap}}

    \caption{
    Schematic representations of the MI-CAP and MI-UAP.
    Different classification regions are represented by the light green and light orange irregular polygons, where the positive region $A^+$ has five example bags $\{ B_1, B_2, \dots, B_5 \}$ and the negative region $A^-$ contains four bags $\{ B_6, B_7, \dots, B_9 \}$:
    (a) The goal of MI-CAP is to generate a minimal perturbation $\epsilon_i$ for each bag $B_i$ in order to remove it from its original classification region.
    (b) The goal of MI-UAP is similar to that of MI-CAP, except that it only generates a universal perturbation for all bags.
    }
    \label{fig: cap_uap}
\end{figure*}

\begin{figure}[!htb]
\centering
\includegraphics[width=0.85\linewidth]{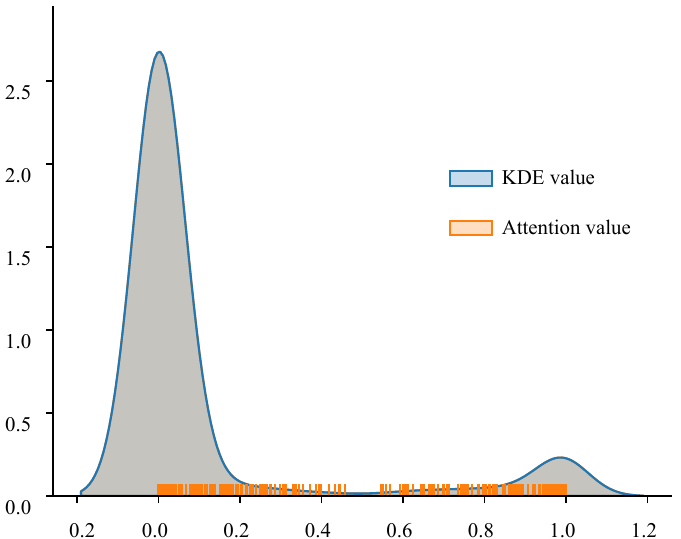}
\caption{
The distribution and KDE results of the attention values.
}
\label{fig: kde}
\end{figure}

Lines 1--2 are the initialization of MI-CAP.
This stage initializes the perturbation and the gradient vector, and get the predicted bag label $\tilde{y}_i$ based on $B_i$ and $\epsilon_i^0$.

Lines 3--18 are the generation of the customized perturbation.
Its critical steps are the computation and updating of the gradient, together with fine-tuning the $\epsilon_i$ by assessing the probability $p_{ic}$ and $\Delta\epsilon_i$.
This tries to ensure that MI-CAP generates a perturbation within a finite number of iterations $L_1$.
The generated perturbation is small enough not to affect human judgment while fooling the MIL learner by making the bag deviate from its original classification region.
A schematic representation of the proposed MI-CAP is shown in Figure \ref{fig: cap}.

It should be emphasized that the perturbation generation strategy of DeepFool is based on identical-resolution images, which is unapplicable in MIL because the number of instances in each bag is not necessarily equal.
But, letting $n_i = n_j, i, j \in [1..N], i \neq j$ in real-world applications is impractical since MIL's flexibility would be lost.
That is, we need some additional operations to unify all gradients into the same shape since the ones corresponding to different bags usually have different shapes.

Therefore, we offer two solutions to the above problem:
The first averages each row of $\hat{G}_i$ and $\tilde{G}_i$, as in Lines 8--9, and the second skillfully makes use of the attention-based MIL methods \cite{Ilse:2018:21272136,Konstantinov:2022:123}, as in Lines 11--12.
Specifically, such methods can return the attention weight $A_i = \{\alpha_{i1}, \alpha_{i2}, \dots, \alpha_{in_i}\}$ related to $B_i$ while obtaining $\tilde{y}_i$, where $\alpha_{ij}$ is the attention value of the instance $\mathbf{x}_{ij}$.
$A_i$ reflects how much the instances in the bag contribute to the prediction.
Taking the example of having only one positive instance in the bag, in an ideal case, one of the values in $A_i$ might be much larger than the other values in $A_i$.
To test this conjecture, the attention values of ABMIL on the MNIST data set with 50 training epochs are gathered, i.e., $\mathcal{A} = \bigcup_{i=1}^{N} A_i$.
The specific parameter settings will be explained in Section \ref{sec: experimental_setups}.
The distribution and the kernel density estimation (KDE) results of $\mathcal{A}$ are shown in Figure \ref{fig: kde}.
The results show that the vast majority of instances have attention values that are close to zero, suggesting the gradients associated with them may not be as significant.

As a result, $A_i$ can be utilized to assist in identifying the instance in which the current bag, when left undisturbed, contributes the most to the bag label.
The index of such an instance is
\begin{equation}\label{eq: ins_idx}
\iota_i = \argmax_j \alpha_{ij},
\end{equation}
and only the gradient of the instance indicated by $\iota_i$ will be used when updating $\epsilon_i$.

\begin{algorithm}[!htb]
\caption{MI-UAP}\label{ag: uap}
 \KwData{
 $\ $\\
 $\quad$Bag $B_i$;\\
 $\quad$MIL neural network $\mathcal{F}$;\\
 $\quad$Number of iterations $L_1$ and $L_2$;\\
 $\quad$Gradient computation mode: ``ave'' or ``att'';\\
 $\quad$Fooling rate threshold $\delta$;\\
 $\quad$Magnitude control parameter $\xi$;}
 \KwResult{
 $\ $\\
 $\quad$Universal perturbation $\epsilon$;}
 Initialize the perturbation $\epsilon^l = \mathbf{0}$ and the fooling rate $r^l = -1$, where $l = 0$\;
 \While{$r^l < \delta$ and $l < L_2$ }{
  $l$++\;
    \For{$B_i \in \mathcal{S}$}{
        \If{$\tilde{y}_i == \hat{y}_i$}{
            $\epsilon_i^l =$ MI-CAP($B_i, \mathcal{F}, L_1$)\;
            $\Delta \epsilon^l =$ Eq. \eqref{eq: extra_goal} with $\epsilon_i^l$\;
            $\epsilon^l \leftarrow \epsilon^l + \Delta \epsilon^l$\;
            Project $\epsilon^l$ via Eq. \eqref{eq: project} with $\xi$\;
        }
    }
    $r^l =$ Eq. \eqref{eq: goal}\;
    \If{$r^l > r^{0}$}{
        $r^0 = r^l$\;
        $\epsilon^0 = \epsilon^l$\;
    }
 }
 \Return $\epsilon^l$\;
\end{algorithm}

\subsection{Universal Adversarial Perturbation (UAP)}

MI-UAP is an extension of MI-CAP that iteratively utilizes the perturbations produced by the latter to obtain a universal perturbation.
This necessarily sacrifices some performance, since customized perturbations always try to fool the MIL learner into making an incorrect judgment about each bag.
But the universal perturbation also implies flexibility, generalizability, and ease of storage, which may have greater applications in the future, such as intrinsic knowledge obtained from the data set.
\cite{Moosavi:2017:17651773,Wang:2022:122,Zhao:2022:17} and others also established the existence of such a perturbation based on traditional neural networks.

Therefore, we aim to find a universal perturbation $\epsilon$ that satisfies:
\begin{equation}\label{eq: goal_uap}
    \begin{aligned}
        \argmin_\epsilon \| \epsilon \|_2,\quad
        \text{s.t.}\quad
        \left\{
        \begin{array}{l}
            Eq. \eqref{eq: goal} \geq \delta;\\
            \| \epsilon \|_\infty \leq \xi,\\
        \end{array}
        \right.
    \end{aligned}
\end{equation}
where $\delta > 0$ is the fooling rate threshold parameter and $\xi > 0$ is the magnitude control parameter of $\epsilon$.
Similar to MI-CAP, a schematic representation of MI-UAP is shown in Figure \ref{fig: uap}, and its solution is shown in Algorithm \ref{ag: uap}.

Line 1 starts MI-UAP by setting the universal perturbation $\epsilon^0 = \mathbf{0}$ and the fooling rate $r^0 = 0$.

Lines 2--17 correspond to the entire optimization process, which ends when the predetermined threshold $\delta$ or the number of iterations $L_2$ is reached.
In the algorithm's main loop, we will iterate over each bag $B_i$ in the data set $\mathcal{S}$, as finding the smallest perturbation $\epsilon_i^l$ for each bag is one of the critical phases.
This is insufficient, though, as we require a perturbation that is as effective for all bags as possible.
As a result, using $\epsilon_i^l$ as the basis, we need to find an extra perturbation $\Delta \epsilon_i$ that satisfies:
\begin{equation}\label{eq: extra_goal}
\Delta \epsilon^l = \argmin_\varsigma \| \varsigma \|_2,\ \ \text{s.t.}\ \ \mathcal{F}(B_i \oplus ( \epsilon^l+\varsigma)) \neq \hat{y}_i.
\end{equation}
Once we have $\Delta \epsilon^l$, we will utilize it to update $\epsilon^l$ as stated in Line 8.
In addition, the updated perturbation might not satisfy the constraints in Equation \eqref{eq: goal_uap}, for which we need to project it into the new feature space:
\begin{equation}\label{eq: project}
\epsilon^l = \argmin_\varsigma \| \epsilon^l - \varsigma \|_2,\quad\text{s.t.}\quad\| \varsigma \|_2 \leq \xi.
\end{equation}
Note that Eq. \eqref{eq: project} is also used to project the perturbation generated by MI-CAP to adjust its magnitude.
At last, the perturbation $\epsilon$ with the highest fooling rate $r$ is returned.

\subsection{Discussion and Extension}

For a more in-depth investigation of MI-CAP and MI-UAP, here we will discuss their time complexities and extend them to more complex cases with the bag of images.

\begin{prop}\label{property: complexity}
The time complexity for MI-CAP and MI-UAP to generate perturbations for all bags is $O(d^2\Sigma)$ and $O(Nd^2\Sigma)$, respectively, where $d$ is the dimension, $\Sigma = \sum_i n_i$ is the total number of instances, and $N$ is the number of bags.
\end{prop}

\begin{proof}\let\qed\relax
For MI-CAP, the gradient calculation in its main loop, which is directly tied to the complexity of the network structure given $\mathcal{F}$, determines the majority of the time complexity of this algorithm.
Using DSMIL \cite{Li:2021:1431814328} as an example, which has a more complex network structure, it mostly consists of linear transformations and convolution operations.
DSMIL therefore costs $O(n_id^2L_1)$ to handle a bag.
Lines 14--16 are primarily composed of simple addition and subtraction operations, whereas Line 17 solely employs the pretrained $\mathcal{F}$ to predict the labels of perturbed bags $B_i \oplus \epsilon_i$.
When contrasting these costs with $n_i d^2$, they can be disregarded.
Generally, we have $L_1 \ll \Sigma$, where $\Sigma = \sum_i n_i$.
Therefore, the time complexity of MI-CAP applied to the whole data set is $O(d^2L_1\sum_i n_i) = O(d^2\Sigma)$.

For MI-UAP, its main loop is itself an operation on all bags, which costs $O(d^2L_2N\Sigma)$ and can be simplified to $O(N d^2\Sigma)$.
Despite taking longer than MI-CAP to generate the perturbation, MI-UAP only costs $O(1)$ during inference to a given bag rather than $O(n_id^2L_1)$.
\end{proof}

The perturbation generation of MI-CAP and MI-UAP is carried out on the premise that the instance in the bag is a vector, which does not satisfy the more complex classification of the bag structure mentioned in \cite{Ilse:2018:21272136,Konstantinov:2022:123}.
It is therefore necessary to extend these two algorithms to enhance their adaptability in the above scenarios.

Let $\mathcal{B}_i = \{ I_{i1}, I_{i2}, \dots, I_{in_i} \}$ be a bag, where $I_{ij}$ is the $j$-th image of $\mathcal{B}_i$.
The objective in Eq. \eqref{eq: goal} will be rewritten as
\begin{equation}
\max_{\mathcal{B}_i \sim \mu} \frac{\sum_i \mathbb{I} (\mathcal{F}' (\mathcal{B}_i \oplus \epsilon') \neq \hat{y}_i)}{\sum_i \mathbb{I} (y_i == \hat{y}_i)},
\end{equation}
where $\mathcal{B}_i \oplus \epsilon' = \{ I_{i1} + \epsilon', I_{i2} + \epsilon', \dots, I_{in_i} + \epsilon' \}$ and $\epsilon'$ has the same amount of rows and columns as $I_{ij}$.
Here, $\mathcal{F}'$ denotes the addition of a convolutional block to the network in order to learn the vector representation of the image, and the specific structure is taken from \cite{Ilse:2018:21272136}.

\section{Experiments}

In this section, we will utilize MI-CAP and MI-UAP to expose the vulnerability of the SOTA MIL neural networks, which are conducted on three image databases, two VAD data sets, and five benchmark data sets.

\subsection{Experimental Setups}\label{sec: experimental_setups}

Five networks were selected to evaluate our methods, including ABMIL and GAMIL \cite{Ilse:2018:21272136}, LAMIL \cite{Shi:2020:57425745}, DSMIL \cite{Li:2021:1431814328}, and MAMIL \cite{Konstantinov:2022:123}.
With the exception of the training epoch and learning rate, we exactly follow the original paper's description for their parameter setups.
For consistency, these two parameters are set to be $50$ and $0.0001$ for all algorithms, respectively.
For MI-CAP and MI-UAP, we set $L_1=10$, $L_2=50$, $\delta=0.5$, and $\xi$ was enumerate in $\{0.01, 0.05, 0.1, 0.2, \dots, 0.5, 1\}$.
The modes they use are denoted by ``ave'' and ``att".
The reason we set $\delta$ this way is that, assuming the data set is class-balanced at this point, only half the fooling rate is required to fool the judgement of the MIL learner.
And the reason for $\xi$ is that we normalize all of the data sets, so the setting of it indirectly reflects how perturbation and data space are related.
And we would select an appropriate $\xi$ before comparing performance, as a too large $\xi$ will distort the bags too much and render the perturbation worthless.

Three famous image databases were used to evaluate the effectiveness and the generalizability of our methods, including MNIST \cite{Lecun:1998:mnist}, CIFAR10 \cite{Krizhevsky:2009:learning}, and STL10 \cite{Coates:2011:215223}.
These databases are not applicable to MIL scenarios and require some special strategies to convert into data sets with MIL data structures.
Thanks to the contribution of \cite{Ilse:2018:21272136} in this regard, the following are the steps to generate MIL data:
1) Specify the number of bags, the range of bag sizes, and the target class of the database (here we set the target class to $9$).
The image from the target class will be used as the positive instance, which determines the label of the bag;
2) Generate the size of each bag at random, then choose the required number of images from the database.
If an image from the target class is present, the label is positive; otherwise, it is negative;
and
3) Return the data set once sufficient bags are generated.

The data sets produced with the above three image databases closely mimic the MIL application scenarios.
However, these data sets were only generated based on the MIL assumption, and the actual application scenarios may be much more complex.
In this regard, we employ two VAD data sets to validate the proposed algorithms.
Specifically, ShanghaiTech \cite{Liu:2018:63566545,Zhong:2019:12371246} and UCF-Crime \cite{Sultani:2018:64796488} are a medium-scale data set from fixed-angle street video surveillance and a large-scale data set from real-world street and indoor surveillance cameras.
Because there are typically many frames in a single video and supervision information is only provided at the video-level, VAD is a typical weak supervision problem.
Therefore, we use the setting of \cite{Tian:2021:49754986}, split each video in the aforementioned data sets into multiple clips, treat each clip as a bag, and use I3D \cite{Kay:2017:122} for preprocessing to help the model learn video features more effectively.

Furthermore, MI-CAP and MI-UAP cannot be used to directly fool traditional MIL methods because most of them do not involve gradient computation but rather certain methods based on key samples, support vectors, or others.
So, we use the universal perturbation to attack traditional MIL methods, including BAMIC \cite{Zhang:2009:4768}, miFV and miVLAD \cite{Wei:2017:975987}, ELDB \cite{Yang:2021:112}, and AEMI \cite{Yang:2022:109121}.
The data sets used both are well-known benchmark drug activity prediction \cite{Dietterich:1997:3171} and image classification data sets \cite{Andrews:2002:561568,Li:2008:9851002}.
The parameter settings of these five methods are referenced to \cite{Yang:2022:109121}.

\begin{figure*}[!htb]
    \centering
    \subfigure[ABMIL]{\includegraphics[width=0.33\hsize]{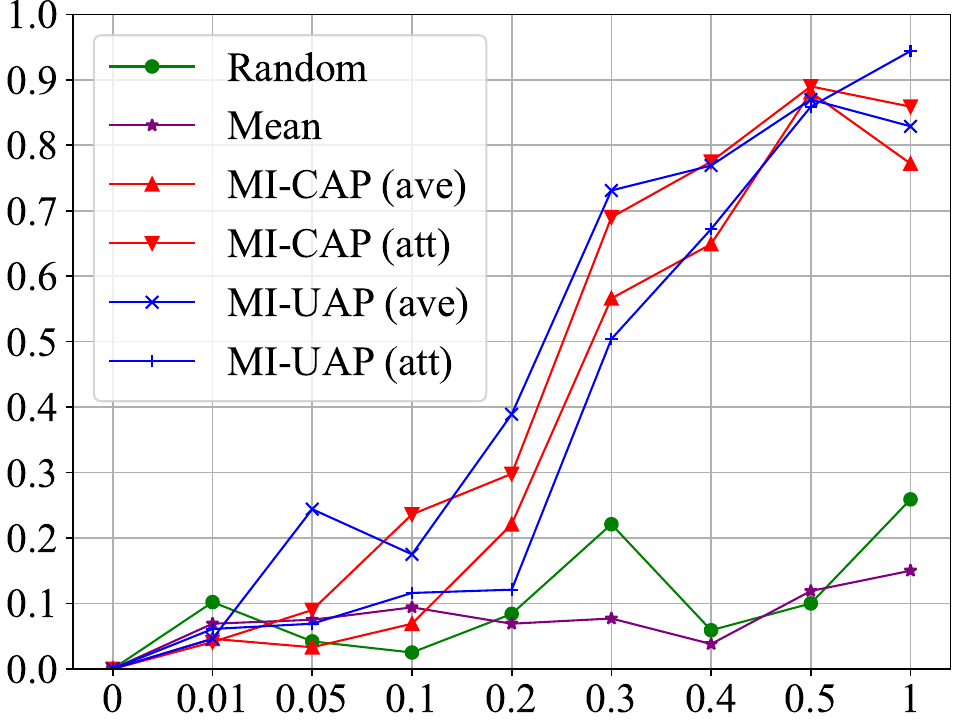}}
    \subfigure[GAMIL]{\includegraphics[width=0.33\hsize]{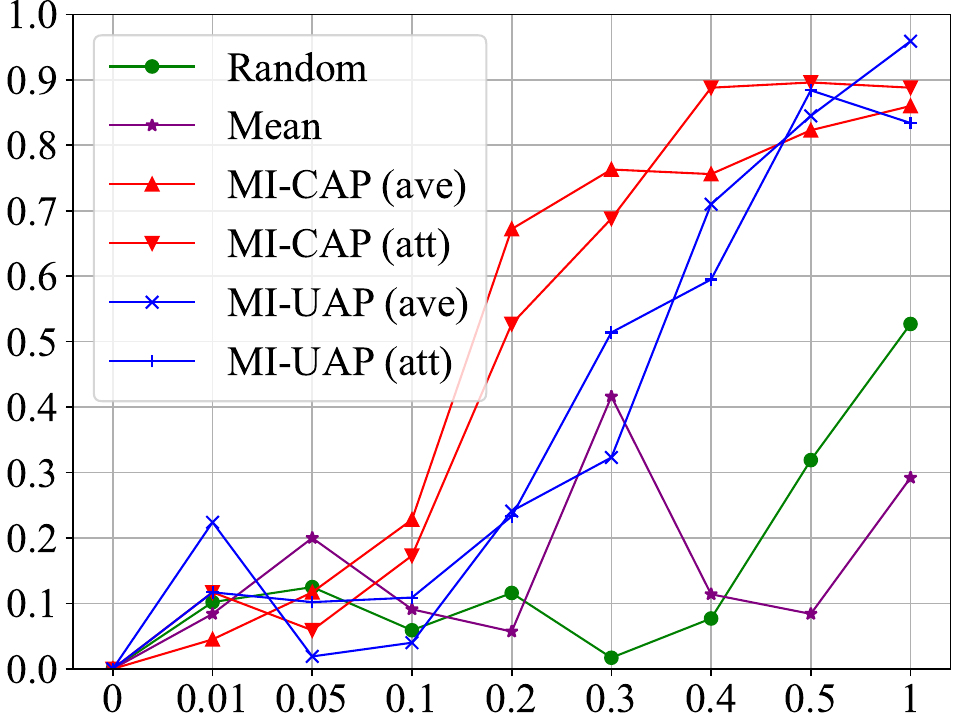}}
    \subfigure[LAMIL]{\includegraphics[width=0.33\hsize]{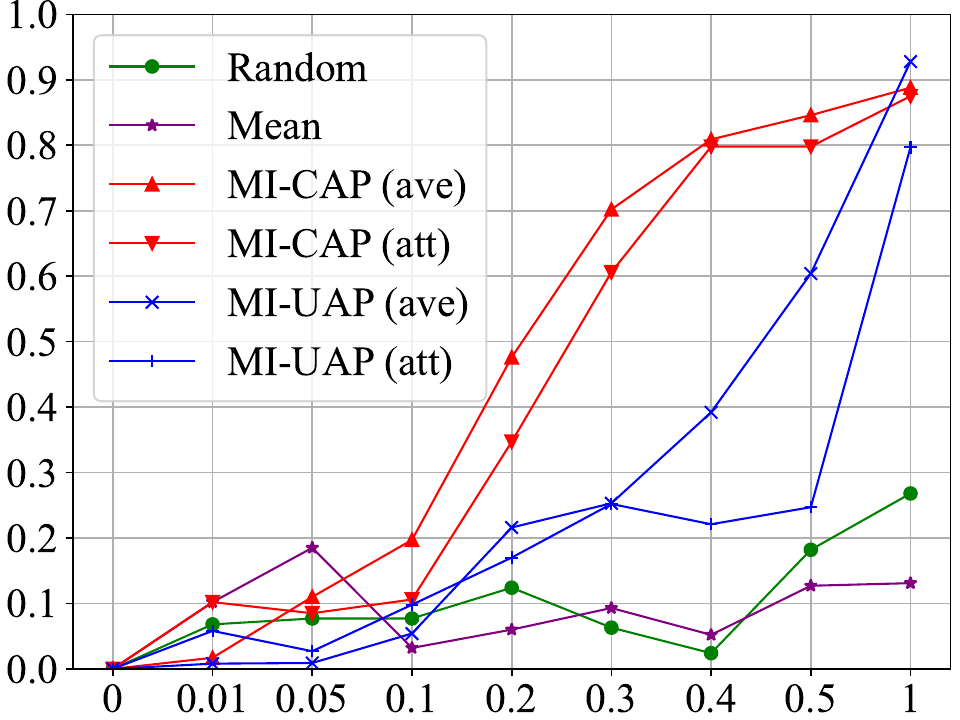}}

    \subfigure[DSMIL]{\includegraphics[width=0.33\hsize]{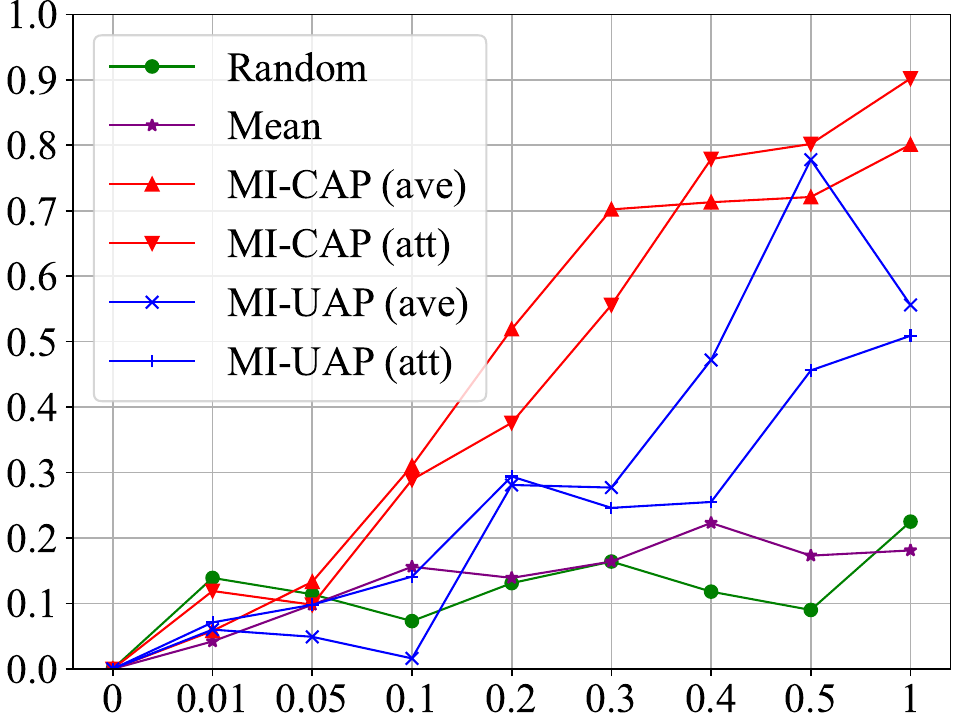}}
    \subfigure[MAMIL]{\includegraphics[width=0.33\hsize]{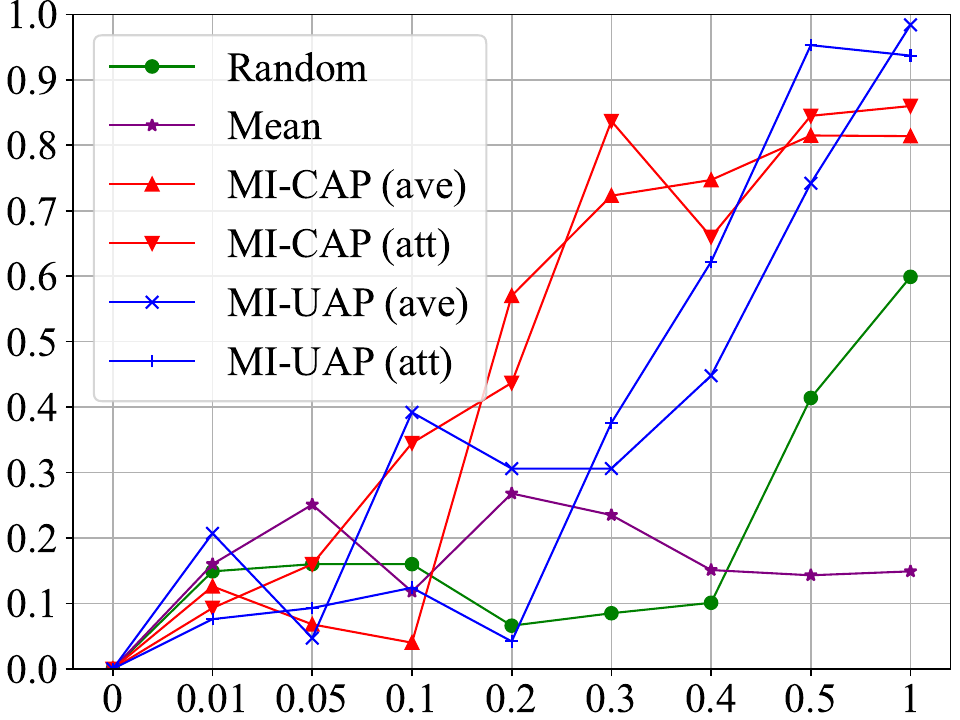}}

    \caption{
    Parameter analysis of the projected magnitude control parameter $\xi$ under the recall metric on the MNIST data set.
    $\xi=0$ means no perturbation is added.}
    \label{fig: xi_ratio}
\end{figure*}

\begin{figure*}[!htb]
    \centering
    \subfigure[MI-UAP (ave)]{\includegraphics[width=0.48\hsize]{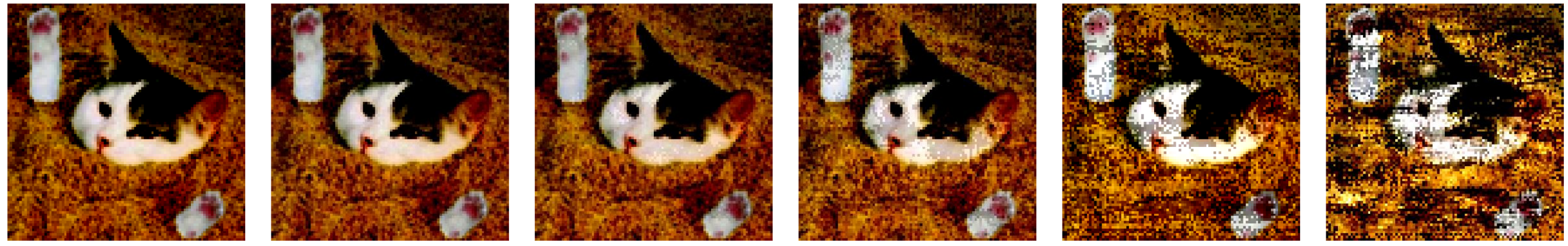}}
    \hspace{0.03\hsize}
    \subfigure[MI-UAP (att)]{\includegraphics[width=0.48\hsize]{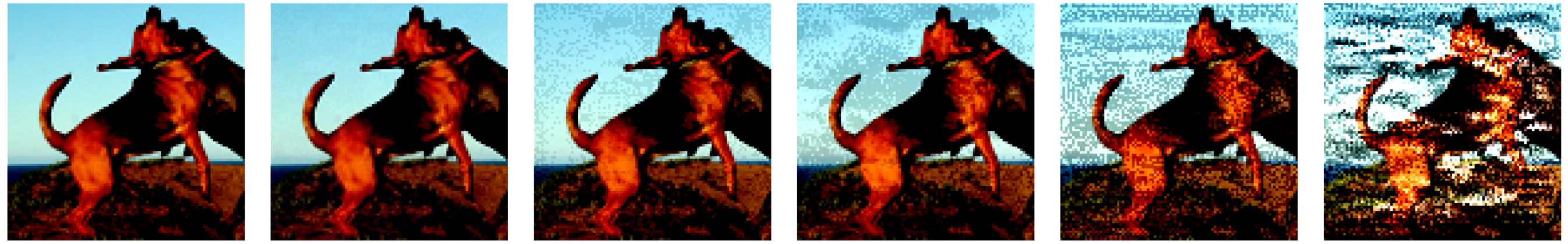}}

    \caption{
    Images with varying magnitudes of perturbations generated by ABMIL on the STL10 data set, where perturbation increases from left to right, corresponding to $0, 0.01, 0.05, 0.1, 0.2$, and $0.3$, respectively.
    }
    \label{fig: xi_increase}
\end{figure*}

\begin{table*}[!htb]
\caption{Performance comparison for network-based classification on three databases using seven perturbation strategies ($\xi = 0.2$).}
\label{table: performance_nn}
\centering
\resizebox{\textwidth}{!}{
\begin{tabular}{|c|cc||ccccc|}
\toprule
Data set                      & Evaluator             & Perturbation& ABMIL \cite{Ilse:2018:21272136}               & GAMIL \cite{Ilse:2018:21272136}
                                                                    & LAMIL \cite{Shi:2020:57425745}                & DSMIL \cite{Li:2021:1431814328}               & MAMIL \cite{Konstantinov:2022:123}\\
\midrule
\multirow{10}{*}{MNIST}       & \multirow{7}{*}{Acc}  & No          & $0.952\pm0.011$                               & $0.972\pm0.018$
                                                                    & $0.960\pm0.014$                               & $0.956\pm0.017$                               & $0.968\pm0.011$\\
                              &                       & Random      & $0.044\pm0.041$                               & $0.088\pm0.033$
                                                                    & $0.104\pm0.009$                               & $0.076\pm0.040$                               & $0.048\pm0.024$\\
                              &                       & Mean        & $0.020\pm0.030$                               & $0.056\pm0.036$
                                                                    & $0.040\pm0.024$                               & $0.076\pm0.014$                               & $0.140\pm0.039$\\
                              &                       & MI-CAP (ave)& $0.224\pm0.065$                               & $\ \ \ 0.496\pm0.086\downarrow$
                                                                    & $\ \ \ 0.364\pm0.071\downarrow$               & $\ \ \ 0.492\pm0.075\downarrow$               & \ \ \ $0.312\pm0.130\downarrow$\\
                              &                       & MI-CAP (att)& $0.208\pm0.112$                               & $0.364\pm0.104$
                                                                    & $0.288\pm0.090$                               & $0.288\pm0.033$                               & $0.276\pm0.090$\\
                              &                       & MI-UAP (ave)& $\ \ \ 0.320\pm0.039\downarrow$               & $0.124\pm0.091$
                                                                    & $0.112\pm0.061$                               & $0.020\pm0.026$                               & $0.156\pm0.204$\\
                              &                       & MI-UAP (att)& $0.060\pm0.039$                               & $0.144\pm0.059$
                                                                    & $0.088\pm0.039$                               & $0.144\pm0.108$                               & $0.020\pm0.030$\\
                              \cdashline{2-8}[1pt/0pt]
                              & \multirow{7}{*}{Rec}  & No          & $0.944\pm0.019$                               & $0.942\pm0.037$
                                                                    & $0.935\pm0.023$                               & $0.943\pm0.021$                               & $0.946\pm0.021$\\
                              &                       & Random      & $0.076\pm0.070$                               & $0.116\pm0.064$
                                                                    & $0.099\pm0.006$                               & $0.118\pm0.090$                               & $0.044\pm0.036$\\
                              &                       & Mean        & $0.069\pm0.051$                               & $0.057\pm0.062$
                                                                    & $0.035\pm0.037$                               & $0.126\pm0.056$                               & $0.246\pm0.080$\\
                              &                       & MI-CAP (ave)& $0.221\pm0.074$                               & $\ \ \ 0.672\pm0.140\downarrow$
                                                                    & $\ \ \ 0.476\pm0.129\downarrow$               & $\ \ \ 0.519\pm0.088\downarrow$               & $\ \ \ 0.570\pm0.317\downarrow$\\
                              &                       & MI-CAP (att)& $0.298\pm0.169$                               & $0.527\pm0.195$
                                                                    & $0.322\pm0.108$                               & $0.363\pm0.102$                               & $0.415\pm0.120$\\
                              &                       & MI-UAP (ave)& $\ \ \ 0.389\pm0.105\downarrow$               & $0.241\pm0.196$
                                                                    & $0.209\pm0.219$                               & $\ \ \ 0.028\pm0.021\uparrow$                 & $0.306\pm0.389$\\
                              &                       & MI-UAP (att)& $0.121\pm0.064$                               & $0.233\pm0.107$
                                                                    & $0.216\pm0.116$                               & $0.281\pm0.221$                               & $0.020\pm0.035$\\
\midrule
\multirow{5}{*}{CIFAR10}      & \multirow{7}{*}{Acc}  & No          & $0.700\pm0.035$                               & $0.640\pm0.037$
                                                                    & $0.624\pm0.043$                               & $0.676\pm0.030$                               & $0.680\pm0.020$\\
                              &                       & Random      & $0.112\pm0.050$                               & $0.024\pm0.022$
                                                                    & $0.004\pm0.020$                               & $0.076\pm0.014$                               & $0.120\pm0.068$\\
                              &                       & Mean        & $0.196\pm0.077$                               & $0.084\pm0.055$
                                                                    & $0.060\pm0.059$                               & $0.126\pm0.023$                               & $0.084\pm0.052$\\
                              &                       & MI-CAP (ave)& $0.132\pm0.084$                               & $0.176\pm0.074$
                                                                    & $\ \ \ 0.268\pm0.027\downarrow$               & $0.128\pm0.093$                               & $0.108\pm0.063$\\
                              &                       & MI-CAP (att)& $0.196\pm0.079$                               & $\ \ \ 0.195\pm0.062\downarrow$
                                                                    & $0.140\pm0.078$                               & $0.144\pm0.039$                               & $0.112\pm0.106$\\
                              &                       & MI-UAP (ave)& $0.136\pm0.086$                               & $0.132\pm0.098$
                                                                    & $0.084\pm0.094$                               & $0.080\pm0.109$                               & $0.076\pm0.108$\\
                              &                       & MI-UAP (att)& $\ \ \ 0.244\pm0.119\downarrow$               & $0.164\pm0.071$
                                                                    & $0.092\pm0.050$                               & $\ \ \ 0.200\pm0.071\downarrow$               & $\ \ \ 0.212\pm0.130\downarrow$\\
\midrule
\multirow{10}{*}{STL10}       & \multirow{7}{*}{Acc}  & No          & $0.804\pm0.033$                               & $0.828\pm0.054$
                                                                    & $0.820\pm0.024$                               & $0.776\pm0.003$                               & $0.808\pm0.090$\\
                              &                       & Random      & $0.112\pm0.018$                               & $0.116\pm0.070$
                                                                    & $0.096\pm0.062$                               & $0.032\pm0.043$                               & $0.108\pm0.048$\\
                              &                       & Mean        & $0.052\pm0.076$                               & $0.196\pm0.122$
                                                                    & $0.068\pm0.038$                               & $0.124\pm0.050$                               & $0.086\pm0.058$\\
                              &                       & MI-CAP (ave)& $0.176\pm0.053$                               & $0.172\pm0.172$
                                                                    & $0.192\pm0.089$                               & $0.155\pm0.062$                               & $\ \ \ 0.228\pm0.081\downarrow$\\
                              &                       & MI-CAP (att)& $0.128\pm0.076$                               & $\ \ \ 0.240\pm0.038\downarrow$
                                                                    & $0.188\pm0.057$                               & $\ \ \ 0.252\pm0.099\downarrow$               & $0.140\pm0.073$\\
                              &                       & MI-UAP (ave)& $0.124\pm0.091$                               & $0.108\pm0.035$
                                                                    & $0.072\pm0.044$                               & $0.061\pm0.052$                               & $0.132\pm0.151$\\
                              &                       & MI-UAP (att)& $\ \ \ 0.212\pm0.104\downarrow$               & $0.176\pm0.142$
                                                                    & $\ \ \ 0.240\pm0.069\downarrow$               & $0.036\pm0.103$                               & $0.065\pm0.087$\\
                              \cdashline{2-8}[1pt/0pt]
                              & \multirow{7}{*}{Rec}  & No          & \multirow{7}{*}{$<0.400$}                     & $0.756\pm0.067$
                                                                    & $0.756\pm0.049$                               & $0.542\pm0.192$                               & \multirow{7}{*}{$<0.400$}\\
                              &                       & Random      &                                               & $0.251\pm0.081$
                                                                    & $0.122\pm0.135$                               & $\ \ \ 0.137\pm0.094\downarrow$               & $\ $\\
                              &                       & Mean        &                                               & $0.220\pm0.173$
                                                                    & $0.289\pm0.128$                               & $0.008\pm0.130$                               & $\ $\\
                              &                       & MI-CAP (ave)&                                               & $0.240\pm0.181$
                                                                    & $0.190\pm0.087$                               & $0.135\pm0.057$                               & $\ $\\
                              &                       & MI-CAP (att)&                                               & $0.195\pm0.057$
                                                                    & $\ \ \ 0.320\pm0.118\downarrow$               & $0.119\pm0.156$                               & $\ $\\
                              &                       & MI-UAP (ave)&                                               & $0.251\pm0.192$
                                                                    & $0.112\pm0.070$                               & $0.115\pm0.083$                               & $\ $\\
                              &                       & MI-UAP (att)&                                               & $\ \ \ 0.361\pm0.077\downarrow$
                                                                    & $0.304\pm0.105$                               & $0.029\pm0.085$                               & $\ $\\
\bottomrule
\end{tabular}}
\end{table*}

\begin{table*}[!htb]
\caption{Performance comparison for network-based classification on three databases using five perturbation strategies ($\xi=0.01$).}
\label{table: performance_vad}
\centering
\resizebox{\textwidth}{!}{
\begin{tabular}{|c|cc||ccccc|}
\toprule
Data set                       & Evaluator             & Perturbation& ABMIL \cite{Ilse:2018:21272136}               & GAMIL \cite{Ilse:2018:21272136}
                                                                     & LAMIL \cite{Shi:2020:57425745}                & DSMIL \cite{Li:2021:1431814328}               & MAMIL \cite{Konstantinov:2022:123}\\
\midrule
\multirow{10}{*}{ShanghaiTech} & \multirow{5}{*}{Acc}  & No          & $0.984\pm0.005$                               & $0.979\pm0.005$
                                                                     & $1.000\pm0.000$                               & $1.000\pm0.000$                               & $0.986\pm0.016$\\
                               &                       & MI-CAP (ave)& $0.602\pm0.026$                               & $\ \ \ 0.633\pm0.065\downarrow$
                                                                     & $\ \ \ 0.611\pm0.028\downarrow$               & $\ \ \ 0.435\pm0.016\downarrow$               & $0.349\pm0.019$\\
                               &                       & MI-CAP (att)& $\ \ \ 0.621\pm0.019\downarrow$               & $0.401\pm0.013$
                                                                     & $0.607\pm0.034$                               & $0.429\pm0.013$                               & $0.355\pm0.026$\\
                               &                       & MI-UAP (ave)& $0.402\pm0.012$                               & $0.398\pm0.019$
                                                                     & $0.449\pm0.022$                               & $0.424\pm0.017$                               & $0.378\pm0.058$\\
                               &                       & MI-UAP (att)& $0.401\pm0.013$                               & $0.439\pm0.057$
                                                                     & $0.430\pm0.025$                               & $0.414\pm0.018$                               & $\ \ \ 0.415\pm0.005\downarrow$\\
                               \cdashline{2-8}[1pt/0pt]
                               & \multirow{5}{*}{Rec}  & No          & $0.940\pm0.021$                               & $0.921\pm0.019$
                                                                     & $1.000\pm0.000$                               & $1.000\pm0.000$                               & $0.946\pm0.061$\\
                               &                       & MI-CAP (ave)& $0.307\pm0.022$                               & $0.279\pm0.031$
                                                                     & $0.320\pm0.019$                               & $\ \ \ 0.822\pm0.031\downarrow$               & $0.660\pm0.030$\\
                               &                       & MI-CAP (att)& $0.298\pm0.016$                               & $\ \ \ 0.752\pm0.022\downarrow$
                                                                     & $0.307\pm0.036$                               & $0.810\pm0.025$                               & $0.660\pm0.022$\\
                               &                       & MI-UAP (ave)& $0.746\pm0.017$                               & $0.721\pm0.047$
                                                                     & $0.787\pm0.027$                               & $0.800\pm0.033$                               & $\ \ \ 0.800\pm0.094\downarrow$\\
                               &                       & MI-UAP (att)& $\ \ \ 0.651\pm0.022\downarrow$               & $0.674\pm0.113$
                                                                     & $\ \ \ 0.794\pm0.037\downarrow$               & $0.754\pm0.023$                               & $0.790\pm0.027$\\
\midrule
\multirow{10}{*}{UCF-Crime}    & \multirow{5}{*}{Acc}  & No          & $0.991\pm0.002$                               & $0.983\pm0.007$
                                                                     & $0.997\pm0.002$                               & $0.999\pm0.001$                               & $0.991\pm0.003$\\
                               &                       & MI-CAP (ave)& $0.498\pm0.015$                               & $\ \ \ 0.504\pm0.005\downarrow$
                                                                     & $0.510\pm0.012$                               & $\ \ \ 0.695\pm0.277\downarrow$               & $0.488\pm0.024$\\
                               &                       & MI-CAP (att)& $0.489\pm0.023$                               & $0.470\pm0.026$
                                                                     & $0.498\pm0.008$                               & $0.685\pm0.234$                               & $0.485\pm0.012$\\
                               &                       & MI-UAP (ave)& $\ \ \ 0.507\pm0.014\downarrow$               & $0.495\pm0.019$
                                                                     & $0.507\pm0.007$                               & $0.527\pm0.017$                               & $0.498\pm0.005$\\
                               &                       & MI-UAP (att)& $0.501\pm0.007$                               & $0.488\pm0.003$
                                                                     & $\ \ \ 0.515\pm0.009\downarrow$               & $0.526\pm0.022$                               & $\ \ \ 0.504\pm0.007\downarrow$\\
                               \cdashline{2-8}[1pt/0pt]
                               & \multirow{5}{*}{Rec}  & No          & $0.990\pm0.002$                               & $0.977\pm0.012$
                                                                     & $0.995\pm0.002$                               & $0.999\pm0.002$                               & $0.990\pm0.003$\\
                               &                       & MI-CAP (ave)& $0.483\pm0.015$                               & $0.459\pm0.006$
                                                                     & $0.485\pm0.009$                               & $0.303\pm0.279$                               & $0.485\pm0.025$\\
                               &                       & MI-CAP (att)& $0.490\pm0.021$                               & $0.484\pm0.024$
                                                                     & $0.494\pm0.005$                               & $0.297\pm0.024$                               & $0.496\pm0.012$\\
                               &                       & MI-UAP (ave)& $\ \ \ 0.520\pm0.016\downarrow$               & $0.533\pm0.049$
                                                                     & $0.530\pm0.034$                               & $0.550\pm0.014$                               & $0.516\pm0.007$\\
                               &                       & MI-UAP (att)& $0.513\pm0.007$                               & $\ \ \ 0.540\pm0.054\downarrow$
                                                                     & $\ \ \ 0.549\pm0.054\downarrow$               & $\ \ \ 0.570\pm0.044\downarrow$               & $\ \ \ 0.517\pm0.009\downarrow$\\
\bottomrule
\end{tabular}}
\end{table*}

\subsection{Parameter Analysis}

Figures \ref{fig: xi_ratio} shows the parameter analysis experiments for the crucial parameter $\xi$ in MI-CAP and MI-UAP.
The abscissa represents $\xi$, and the ordinate represents the recall value of the classification.
The results show that the MIL networks may be fooled more readily as perturbation strength increases and can even drop recall near zero when perturbation reaches a certain threshold.
This implies that perturbations with high fooling rates are entirely possible, right?
Evidently not, and Figure \ref{fig: xi_increase} clearly refutes the claim that ``perturbations with large fooling rates are also more effective."
It can be seen from this figure that the image will be distorted when the perturbation exceeds a certain magnitude, which is obviously unreasonable.
Therefore, we recommend that the value of $\xi$ for image classification not be more than $0.2$.
Figure \ref{fig: xi_ratio} further demonstrates that MI-CAP (ave), which is followed by MI-CAP (att), has the best overall performance, while MI-UAP (ave) only gains an edge over others on ABMIL.
Meanwhile, a very interesting case is that MI-CAP (ave) gets similar results on GAMIL for $\xi=0.01$ and $\xi=0.2$, while the former has a significantly smaller perturbation magnitude than the latter.
This partially reflects its potential, which we will focus on expanding in the future.
Some other findings are as follows:
a) The perturbation generated by MI-CAP and MI-UAP is not random or irregular, and it can show a better effect with an increase in the perturbation magnitude without taking the distortion of the bags into account;
Perturbation generation mode ave outperforms mode att overall, probably because its generated perturbations are added to each instance in the bag, while mode att only focuses on the gradient of the most important instances;
c) The original bag is only slightly perturbed when $\xi \leq 0.05$;
and
d) Perturbations affect both target and irrelevant background areas, illustrating the need to limit the regions in which they work, something that current algorithms are unable to do.

\subsection{Effectiveness Comparison}

Table \ref{table: performance_nn} displays the findings of the effectiveness comparison of our methods against the ``random'' and ``mean'' approaches introduced in \cite{Moosavi:2017:17651773}.
``No'' specifically refers to no perturbation added, which corresponds to the original classification performance of the MIL neural network.
For each data set under the metric ``No'', the average accuracy (acc) and recall (rec) of five individual experiments under $\xi=0.2$ were given, together with their standard deviation (the value with ``$\pm$''), e.g.,
\begin{equation}
Acc = \frac{1}{N}\sum\limits_{i = 1}^N \mathbb{I} ({y}_i == \hat{y}_i),
\end{equation}
For others, we show a decrease in the algorithm accuracy after adding perturbation, e.g.,
\begin{equation}
Decrease = Acc - \frac{1}{N}\sum\limits_{i = 1}^N \mathbb{I} ({y}_i == \tilde{y}_i).
\end{equation}
For example, the second and third rows of the fourth column indicate that the average accuracy of ABMIL on the MNIST data set is $95.2\%$, which drops by $4.4\%$ after adding random perturbation.
One very important reason is that the accuracy after perturbation cannot fully reflect the performance of the perturbation algorithm because the results of multiple independent experiments of the method without perturbation are typically not equal.
The down arrow indicates the best perturbation algorithm for the current column, and the up arrow denotes that the perturbation improves model accuracy.
Note that the recall for more than three algorithms is below $0.4$ on CIFAR10 under ``No'', hence it is not displayed here.

The results show that MI-CAP and MI-UAP can fool the MIL neural network to a large extent, even bringing the recall on some data sets below $0.6$.
This is sufficient for MIL because adding perturbations to make the negative bag be predicted as positive is inherently ridiculous.
Another good finding is that the algorithm can be very unstable when predicting data with perturbations, which will further lower the confidence of the algorithm.
Meanwhile, the results in Table \ref{table: performance_nn} as well as Figures \ref{fig: xi_ratio} and \ref{fig: xi_increase} comprehensively reveal that our MI-CAP and MI-UAP are significantly different from random or mean perturbations, in terms of both performance and the regularity they exhibit.

VAD is one of the crucial applications that is directly tied to people's work and daily lives.
Criminals might inflict significant property damage if they adopt some techniques to interfere with the process, such as generating specialized perturbations to fool video surveillance.
To think more deeply about adversarial methods, we conducted experiments on such data sets, as shown in Table \ref{table: performance_vad}.
It should be observed that the five MIL neural networks perform poorly on CIFAR10 and STL10;
in particular, the recall value on CIFAR10 is very low, making it difficult to accurately depict MI-CAP and MI-UAP's performance.
Therefore, we show in this table the accuracy and standard deviation of the training data set.
Here, we also did not conduct parameter analysis tests because $\xi = 0.01$ produced good results.
The results reveal that VAD methods require more safety concerns since our techniques perform well on the VAD data sets.

\begin{figure*}[!htb]
    \centering
    \subfigure[Ave]{\includegraphics[width=0.44\hsize]{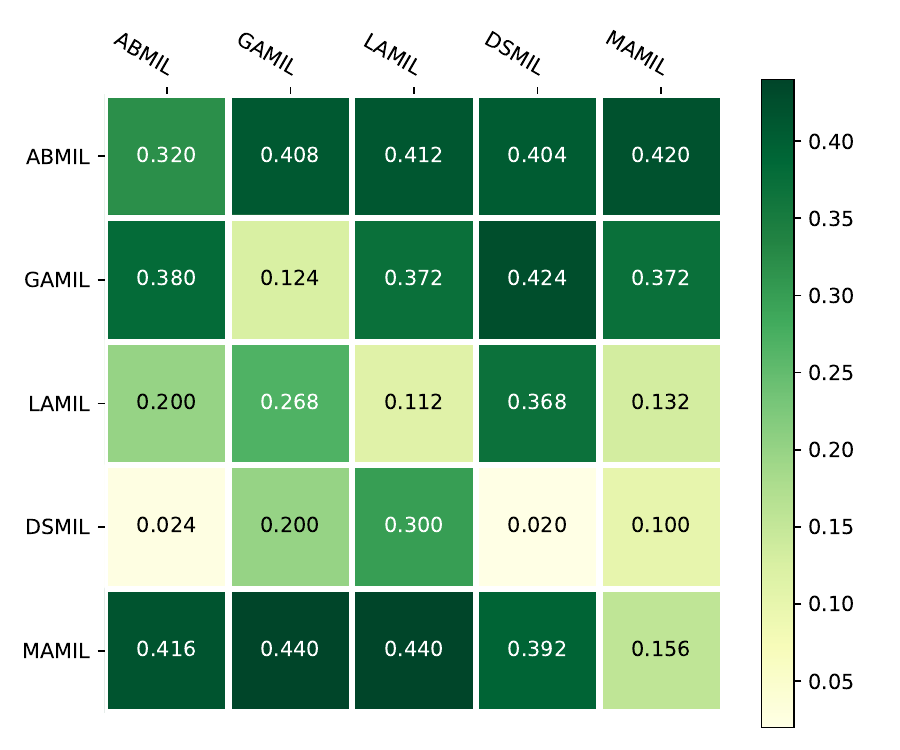}}
    \subfigure[Att]{\includegraphics[width=0.44\hsize]{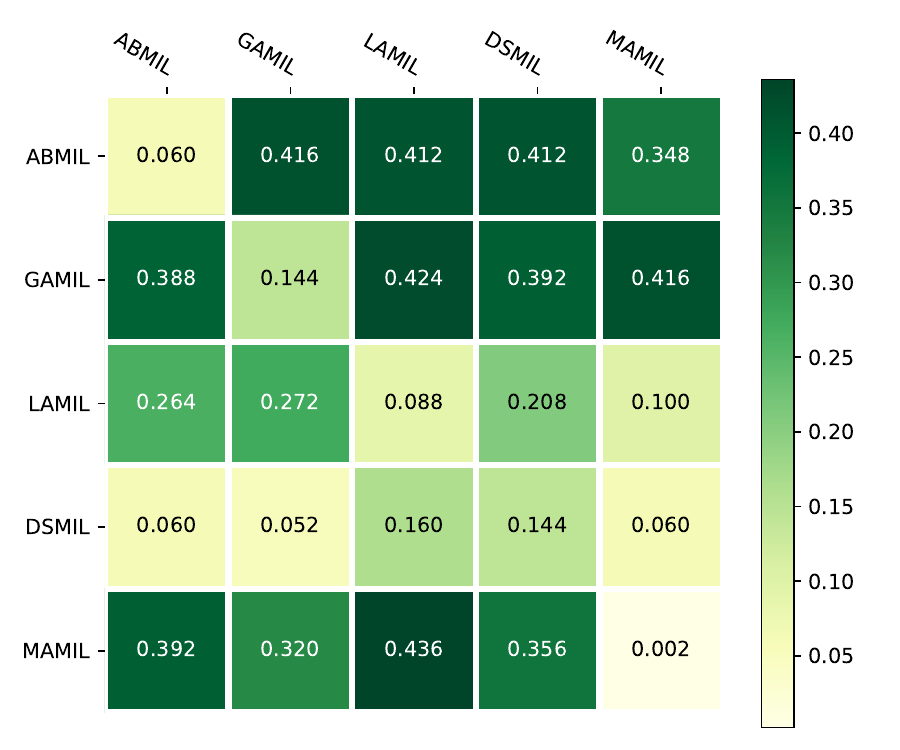}}

    \caption{
    The generalizability experiments of MI-UAP are achieved by applying perturbations generated on one network to other networks ($\xi=0.2$).
    }
    \label{fig: heat}
\end{figure*}

\begin{table*}[!htb]
\caption{Performance comparison for classification on five MIL benchmark data sets using MI-UAP (att) with ABMIL.}
\label{table: performance_traditional}
\centering
\resizebox{\textwidth}{!}{
\begin{tabular}{cccccccc}
\toprule
Data set                      & BAMIC \cite{Zhang:2009:4768}     & miFV \cite{Wei:2017:975987}      & miVLAD \cite{Wei:2017:975987}    & ELDB \cite{Yang:2021:112}        & AEMI \cite{Yang:2022:109121}\\
\midrule
Musk1                         & $0.885\pm0.034 \to 0.860\pm0.011$& $0.887\pm0.025 \to 0.484\pm0.009$& $0.829\pm0.050 \to 0.664\pm0.050$& $0.902\pm0.016 \to 0.853\pm0.012$& $0.867\pm0.019 \to 0.517\pm0.021$\\
Musk2                         & $0.868\pm0.022 \to 0.774\pm0.024$& $0.748\pm0.034 \to 0.622\pm0.007$& $0.788\pm0.041 \to 0.642\pm0.032$& $0.857\pm0.039 \to 0.798\pm0.034$& $0.804\pm0.010 \to 0.384\pm0.009$\\
Elephant                      & $0.844\pm0.013 \to 0.769\pm0.013$& $0.848\pm0.008 \to 0.500\pm0.000$& $0.859\pm0.031 \to 0.420\pm0.020$& $0.843\pm0.012 \to 0.757\pm0.015$& $0.875\pm0.010 \to 0.461\pm0.049$\\
Fox                           & $0.647\pm0.020 \to 0.658\pm0.004$& $0.639\pm0.010 \to 0.500\pm0.000$& $0.611\pm0.020 \to 0.427\pm0.027$& $0.648\pm0.014 \to 0.501\pm0.030$& $0.583\pm0.019 \to 0.423\pm0.022$\\
Tiger                         & $0.797\pm0.009 \to 0.692\pm0.008$& $0.646\pm0.032 \to 0.491\pm0.021$& $0.701\pm0.044 \to 0.427\pm0.037$& $0.767\pm0.013 \to 0.683\pm0.012$& $0.795\pm0.010 \to 0.477\pm0.045$\\
\bottomrule
\end{tabular}}
\end{table*}

\subsection{Generalizability Comparison}

We demonstrate the generalizability of the algorithm from two aspects, including applying perturbations generated on one network to other networks and to traditional MIL methods.

Figure \ref{fig: heat} shows an example of the generalizability of MI-UAP on the MNIST data set.
Each row depicts the result of attacking other algorithms with the perturbation generated by the current algorithm.
The diagonal lines specifically refer to the classification results in Table \ref{table: performance_vad}.
The hue of the square box gets darker as the value increases, signifying a stronger generalizability of the perturbation.
The results demonstrate that perturbations to one algorithm can effectively attack others, possibly with higher fooling rates.
This also illustrates the need for developing MI-UAP from another perspective, as it can generate only one perturbation for one data set, making it easy to store and use right away.

Specifically, we use the perturbations generated by MIL neural networks on the benchmark data sets to fool the traditional MIL methods, as shown in Table \ref{table: performance_traditional}.
The major goal of this experiment is to confirm that using MI-UAP to fool the traditional MIL ones is feasible because adding perturbations to drug molecules is unpractical.
The results show that our method works in most cases, with limited underperformance on BAMIC and ELDB on some data sets.
It seems like a coincidence that both algorithms are bag-based embedding methods.
Therefore, we speculate that the possible reason is that the added perturbation occasionally does not alter the metric relationship based on bag-level distance.

\subsection{Perturbation Defence}

\begin{figure*}[!htb]
    \centering
    \subfigure[Acc]{\includegraphics[width=0.44\hsize]{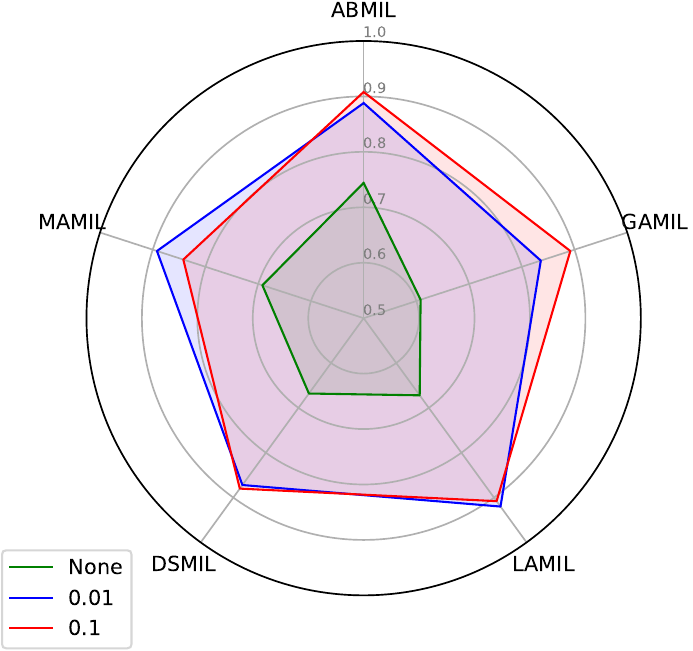}}
    \subfigure[Rec]{\includegraphics[width=0.44\hsize]{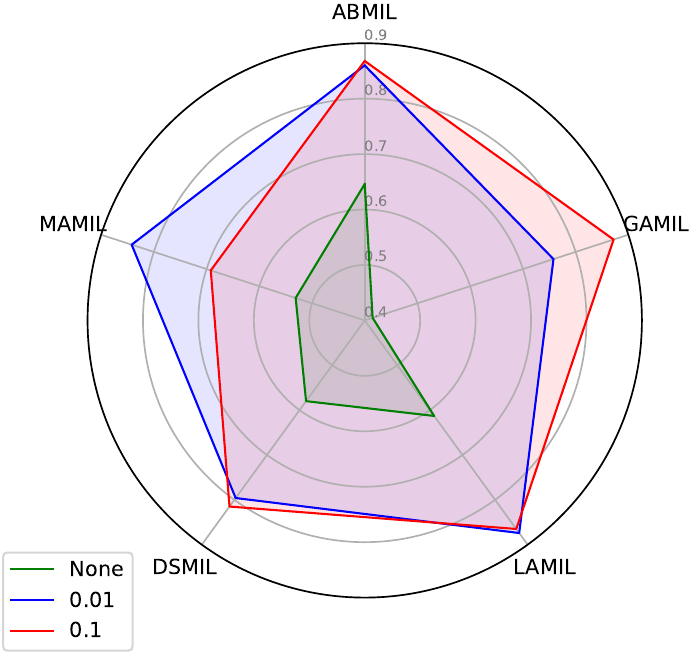}}

    \caption{
    Radar map of the classification results after adding some adversarial examples to the MNIST data set ($\xi = 0.2$).
    The numbers $0.01$ and $0.1$ indicate the percentage of adversarial examples in the training set, while None represents the situation when no adversarial examples are introduced.
    }
    \label{fig: redar}
\end{figure*}

Now that we have designed two methods to expose the vulnerability of the MIL embedding methods, some algorithm protection tactics should be taken into account to ensure the practical use of the MIL methods.
A simple and efficient strategy is to add a few adversarial examples to the training set, as shown in Figure \ref{fig: redar}, which shows the classification results of adversarial examples generated based on MI-CAP (att).
The reason for choosing MI-CAP is that it is more aggressive and has a good ability to fool the classifier.
The findings demonstrate that, although this is not as good as the performance without perturbation, only a small number of adversarial examples are needed to considerably increase the robustness of the classifier.
Why not keep the percentage of adversarial examples rising?
Because we believe that rather than doing something as straightforward as adding a lot of adversarial examples, our future focus should be on enhancing the safety of the algorithm itself.

\section{Conclusions}

In this paper, we have proposed two perturbation generation algorithms, MI-CAP and MI-UAP, to somewhat fool the MIL learner and interpret its vulnerability.
Extensive experiments are conducted to demonstrate the effectiveness and generalizability of the proposed algorithms.
A simple strategy is also proposed to deal with these adversarial perturbations.
This helps us to undersatnd more clearly about the factors that influence the security of an algorithm, and actually provides tactics for enhancing it.
In future, we plan to focus more on developing algorithms that are more secure.
Specifically, in view of developing a more secure MIL algorithm, the following points can be considered.

\begin{enumerate}
  \item
  In the generated perturbations, more attention should be paid to the valuable parts of the bag, such as instances that can trigger the bag label and essential areas within instances.
  \item
  Consider important details or unique circumstances in the actual application scenario to generate more targeted perturbation.
  \item
  Attempts to identify perturbations in input data to rule out or lessen the impact of adversarial perturbations.
\end{enumerate}

\section*{Declaration of Competing Interest}

No conflict of interest exits in the submission of this manus cript, and manuscript is approved by both authors for publication. We declare that the work described was original research that has not been published previously, and not under consideration for publication elsewhere, in whole or in part.

\section*{Acknowledgments}

This work was supported in part by the National Natural Science Foundation of China (62131016) and Open Project of Zhejiang Key Laboratory of Marine Big Data Mining and Application (OBDMA202102).

\end{document}